%% file: AAAI_main.tex
\documentclass[letterpaper]{article} %
\usepackage{aaai2026}  %
\usepackage{times}  %
\usepackage{helvet}  %
\usepackage{courier}  %
\usepackage[hyphens]{url}  %
\usepackage{graphicx} %
\usepackage{natbib}  %
\usepackage{caption} %
\usepackage{newfloat}
\usepackage{listings}

\DeclareCaptionStyle{ruled}{labelfont=normalfont,labelsep=colon,strut=off} %
\usepackage[linesnumbered,ruled,noend]{algorithm2e}  %

\usepackage{xspace}

\usepackage[utf8]{inputenc}
\usepackage{dblfloatfix} 
\usepackage[T1]{fontenc}    %
\usepackage{url}            %
\usepackage{booktabs}       %
\usepackage{amsfonts}       %
\usepackage{nicefrac}       %
\usepackage{microtype}      %

\usepackage{amsmath,amsthm}
\usepackage{subcaption}
\usepackage{graphicx}
\usepackage{color, soul}
\usepackage{comment}
\usepackage{mathtools}
\usepackage{overpic}

\usepackage{enumitem}
\usepackage{svg}
\usepackage{cleveref}
\usepackage{float}

\usepackage [english]{babel}
\usepackage [autostyle, english = american]{csquotes}
\MakeOuterQuote{"}

\usepackage{pgfplots}
\pgfplotsset{compat=1.17}
\usepackage{graphics} %
\usepackage{epsfig} %
\usepackage{amssymb}  %

\title{%
Robust Out-of-Order Retrieval for Grid-Based Storage at Maximum Capacity
}
\author{
    Tzvika Geft,
    William Zhang,
    Jingjin Yu,
    Kostas Bekris
}
\affiliations{
    Computer Science Department, Rutgers University\\
    New Brunswick, NJ, USA
}

\input{macros}

\newcommand{\prevalg}{BaseS\xspace} 
\newcommand{\imps}{\textbf{RobustS}\xspace} 
\newcommand{\impr}{\textbf{ImpR}\xspace} 

\newcommand{\stormr}{{\tt StoRMR}\xspace}
\newcommand{\rstormr}{{\tt R-StoRMR}\xspace}

\newcommand{\storage}{\ensuremath{W}\xspace}

\newcommand{\cols}{\ensuremath{c}\xspace}
\newcommand{\rows}{\ensuremath{r}\xspace}

\newcommand{\nobjs}{\ensuremath{n}\xspace}
\newcommand{\arr}{\ensuremath{\A}\xspace}

\newcommand{\inc}{\ensuremath{A}\xspace}
\newcommand{\dep}{\ensuremath{D}\xspace}

\begin{document}

\maketitle

\input{main_content}

\bibliography{references, mrmp-references, uncertainty_refs}

\end{document}

%% file: macros.tex
\definecolor{gray}{rgb}{0.35,0.35,0.35}
\definecolor{blue}{rgb}{0,0,1}
\definecolor{red}{rgb}{1,0,0}
\definecolor{orange}{rgb}{0.75, 0.4, 0}
\definecolor{green}{rgb}{0.0, 0.5, 0.0}

\newtheorem*{proposition*}{proposition}

\newcommand{\ignore}[1]{}

 \def\A{\mathcal{A}}

\newcommand{\Cpp}{C\raise.08ex\hbox{\tt ++}\xspace}

\newtheorem{problem}{Problem} 
\newtheorem*{problem*}{Problem} 

\newtheorem{theorem}{Theorem}
\newtheorem{corollary}{Corollary}

\newtheorem{prop}{Proposition}

\theoremstyle{definition}
\newtheorem{definition}{Definition}
\newtheorem{remark}{Remark}
\theoremstyle{plain}
\newtheorem{observation}{Observation}

\def\0{\bm{0}}

\makeatletter
\def\thmhead@plain#1#2#3{%
  \thmname{#1}\thmnumber{\@ifnotempty{#1}{ }\@upn{#2}}%
  \thmnote{ {\the\thm@notefont#3}}}
\let\thmhead\thmhead@plain
\makeatother

%% file: main_content.tex
\begin{figure*}[ht]
    \centering   
    \begin{overpic}
    [width=.95\linewidth]{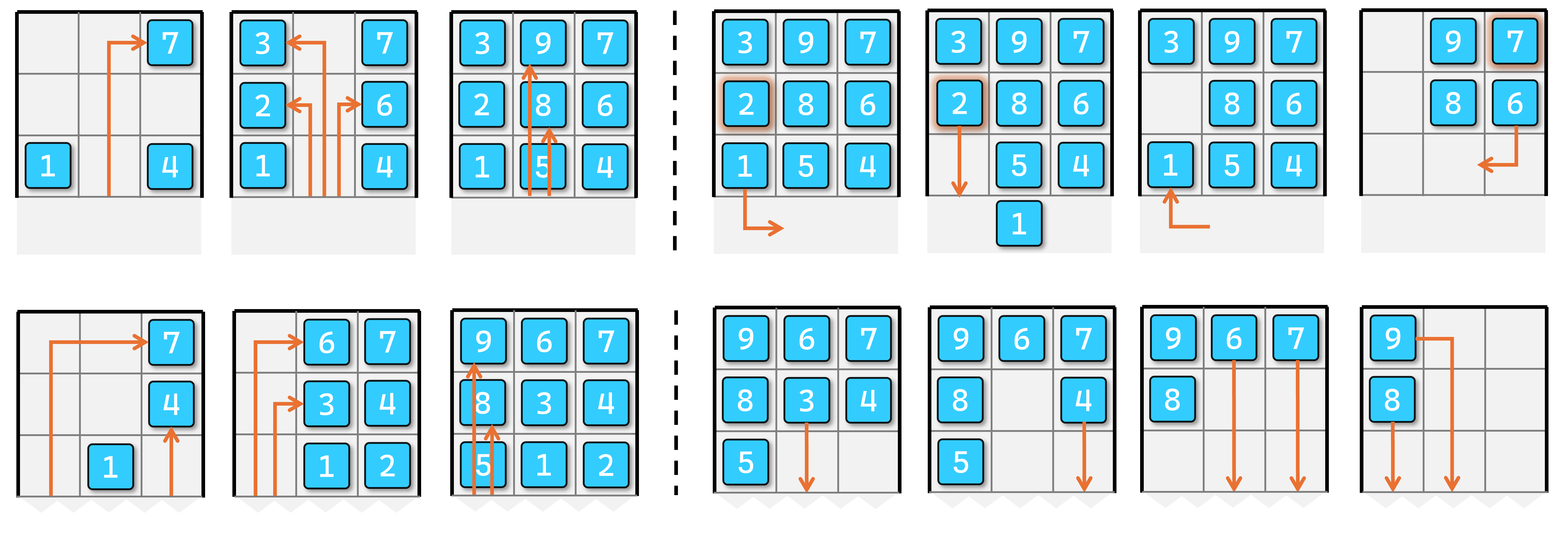}
   \put(6, 16){{\small ($a$)}}
    \put(20, 16){{\small ($b$)}}
    \put(33.6, 16){{\small ($c$)}}
    \put(50, 16){{\small ($d$)}}
    \put(64, 16){{\small ($e$)}}
    \put(77.4, 16){{\small ($f$)}}
    \put(91.6, 16){{\small ($g$)}}
    \put(6, 0.1){{\small ($h$)}}
    \put(20, 0.1){{\small ($i$)}}
    \put(33.6, 0.1){{\small ($j$)}}
    \put(50, 0.1){{\small ($k$)}}
    \put(64, 0.1){{\small ($\ell$)}}
    \put(77.4, 0.1){{\small ($m$)}}
    \put(91.6, 0.1){{\small ($n$)}}
    \end{overpic}
    \caption{
    Consider a $3 \times 3$ storage area $\storage$ accessible only from one side (bottom) that must store 9 loads arriving in the order $\inc = (4,1,7,6,3,2,9,8,5)$ and planned to depart in the order $\dep =(1,2,\ldots,9)$.
    \textbf{Top row:} a solution that avoids relocations if $\dep$ does not change.
    (a) The first three arriving loads, $4,1,7$, are stored.  
    (b) Loads $6,3,2$ are stored.  
    (c) Loads $9,8,5$ are stored.  
    At this point, while the loads can depart according to $\dep$ (not shown), the actual retrieval sequence becomes $\tilde{D} = (2,1,3,5,4,7,6,9,8)$, a slight perturbation of $\dep$, requiring relocations:
    (d) Load 2 is blocked, so 1 is relocated outside $\storage$.  
    (e) Load 2 is retrieved.  
    (f) Load 1 is stored back until it is needed.  
    (g) After $1,3,5,4$ are retrieved, load 7 is next but is blocked by 6, which is relocated within $\storage$.
    \textbf{Bottom row:}
    (h)–(j) Loads are stored according to $\inc$ but in a robust arrangement, as proposed in this work.  
    (k)–(n) In this solution, loads can be retrieved without relocations, not only according to $\dep$, but also under $\tilde{D}$.
}
    \label{fig:intro_ex} 
\end{figure*}
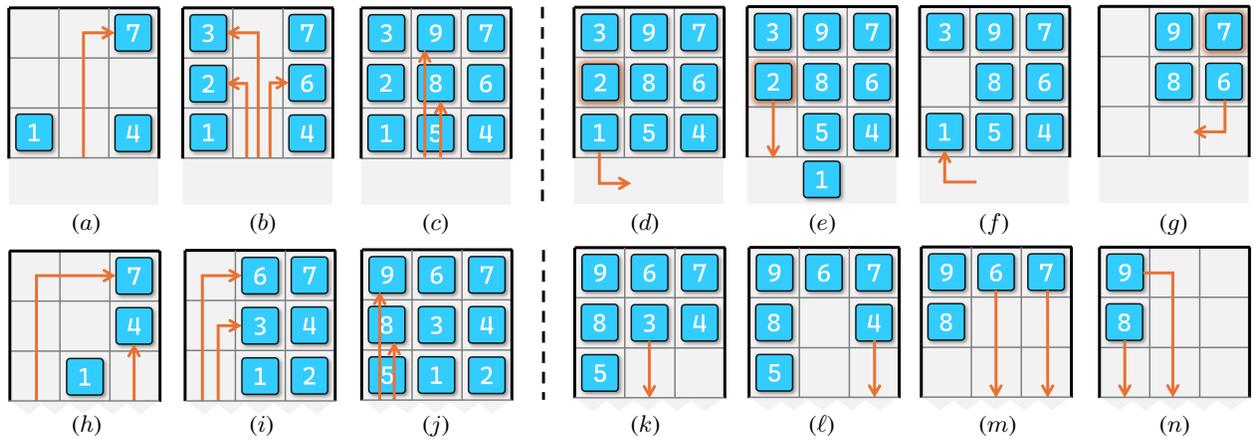

\begin{abstract}
This paper proposes a framework for improving the operational efficiency of automated storage systems under uncertainty.
It considers a 2D grid-based storage for uniform-sized \emph{loads} (e.g., containers, pallets, or totes), which are moved by a robot (or other manipulator) along a collision-free path in the grid. The loads are labeled (i.e., unique) and must be stored in a given sequence, and later be retrieved in a different sequence---an operational pattern that arises in logistics applications, such as last-mile distribution centers and shipyards. The objective is to minimize the load relocations to ensure efficient retrieval. A previous result guarantees a zero-relocation solution for known storage and retrieval sequences, even for storage at full capacity, provided that the side of the grid through which loads are stored/retrieved is at least 3 cells wide. However, in practice, the retrieval sequence can change after the storage phase.
To address such uncertainty, this work investigates \emph{$k$-bounded perturbations} during retrieval, under which any two loads may depart out of order if they are originally at most $k$ positions apart.
We prove that a $\Theta(k)$ grid width is necessary and sufficient for eliminating relocations at maximum capacity.
We also provide an efficient solver for computing a storage arrangement that is robust to such perturbations.
To address the higher-uncertainty case where perturbations exceed $k$, a strategy is introduced to effectively minimize relocations.
Extensive experiments show that, for $k$ up to half the grid width, the proposed storage-retrieval framework essentially eliminates relocations.
For $k$ values up to the full grid width, relocations are reduced by $50\%+$.

\end{abstract}

\section{Introduction}
Modern logistics systems increasingly rely on automation technologies for transporting uniform-sized \emph{loads}, such as containers, pallets, and totes.
Operations at some logistics hubs, especially last-mile and small-scale distribution centers, occur in two distinct phases: first, the storage of incoming loads (such as when a delivery semi-truck arrives), and later, their retrieval for onward transport (such as last-mile, local delivery trucks).  Such scenarios arise in container terminals~\cite{auto-port}, cross-docking facilities~\cite{cross-dock}, automated warehouses with fleets of mobile robots~\cite{wurman2008coordinating} and Automated Storage and Retrieval Systems (AS/RS)~\cite{asrs-survey, yalcin2017multi, carton-seq}.
A central challenge these systems must contend with is trading off between maximizing space utilization and storage/retrieval efficiency, as denser storage necessarily makes arbitrary load access more difficult.

This paper investigates this trade-off in high-density, 2D grid-based storage, akin to Puzzle-Based Storage (PBS)~\cite{gue2007puzzle}.  %
In this setting, each cell of a rectangular $\rows\times\cols$ grid can hold one load, which can be moved by a mobile robot or manipulator along cardinal directions via empty cells. Assuming the grid is accessible for storage and retrieval from only one side,
the following three types of actions are available: \textit{(i)} \emph{storage} of an arriving load, \textit{(ii)} \emph{retrieval} of a departing load, or \textit{(iii) }\emph{relocation (rearrangement)} of a load from one cell to another.
Loads must be stored in a given known order \inc and must then be retrieved according to an \emph{anticipated} retrieval order \dep, which might change, while minimizing relocations.
See \Cref{fig:intro_ex}; a formal problem definition follows below.

Assuming full prior knowledge of the storage and retrieval sequence, prior work shows that rearrangements can be avoided, even when the grid is to be fully occupied, provided that the grid’s open, access side is at least 3 columns wide~\cite{rss25}.
While this result eliminates the aforementioned trade-off under full order observability, in practice, the complete load sequence is generally unavailable due to operational uncertainties. This paper instead asks: \emph{can rearrangements be eliminated or minimized under significant storage/retrieval sequence uncertainty?}
This work provides a positive answer through a two-part framework that combines robust storage and effective retrieval, thereby broadening the applicability of high-density grid-based storage under uncertainty and offering design guidelines.

\noindent\textbf{Contribution.}
This work introduces a novel storage and retrieval problem variant that incorporates uncertainty through \emph{$k$-bounded perturbations}, under which any two loads can depart (or, interchangeably, arrive) out of order if they are planned to depart at most $k$ positions apart.
To solve the problem, two complementary approaches are presented:

\textbf{Robust storage.}
 We generalize deterministic zero-relocation solution conditions to handle uncertainty through \emph{$k$-robust} storage arrangements, which solve the problem with no relocations under $k$-bounded perturbations.
Given $k$, we provide asymptotically tight bounds showing that $\Theta(k)$ columns are necessary and sufficient to find a $k$-robust arrangement.
We also develop a fast solver that finds robust arrangements for $k$ values in line with our bounds.

\textbf{Retrieval strategy.}
As rearrangements may be required as $k$ increases, a load relocation problem arises:
Given a target load for retrieval that is blocked by other loads, compute relocation actions for the blocking loads with the goal of minimizing future relocations.
We propose a greedy approach that prioritizes relocations within the storage area while accounting for future retrievals.

Comprehensive experiments for storage at full capacity show that the two approaches combined significantly outperform baselines, reducing the number of relocations by up to 60-70\%, the usage of a buffer row outside of the storage area, and the distance traveled by loads, while maintaining computational efficiency %
as the grid size grows.

\input{related}

\section{Problem Definition} \label{sec:prob-def}
Consider a rectangular $\rows \times \cols$ grid storage area \storage with \rows rows and \cols columns. The bottom (front) row of \storage is the open side of \storage through which loads are stored/retrieved. 
Denote the columns by $C_1, \ldots, C_{\cols}$ in a left-to-right order.
The loads have distinct labels $1,\ldots, \nobjs \le \rows\cols$.
The \emph{density} of a storage space having $\nobjs$ loads is $\nobjs/(\rows \cols)$.
Unless otherwise stated, assume full capacity storage, i.e., $n = rc$.
Each load occupies exactly one grid cell.
An {\em arrangement} \arr of a set of loads is an injective mapping of loads to grid cells,
i.e., an arrangement specifies a distinct (row, column) pair for each load.
Two loads are \emph{adjacent} in an arrangement if they are located in horizontally or vertically adjacent grid cells.

\textbf{I/O row.}
We assume there is an Input/Output (I/O) row adjacent to the front row of \storage; loads appear/disappear on the I/O before/after storage.
The I/O row also serves as a temporary \emph{buffer} that can be used when relocating loads, e.g., to facilitate access to a load.
The I/O row is not used for storage.
Denote this row as $IO$ and $\storage^+ = W \bigcup IO$.

\textbf{Load movement.}
Each load can be moved by a robot via a path of empty cells along the four cardinal directions (up, down, left, or right).
To pick up a specific target load, the robot must reach the cell occupied by the target load. %
The following types of \emph{actions} are valid:
\begin{itemize}[leftmargin=3mm]
    \item Storage: A load can be \emph{stored} in an empty cell $v \in \storage$ via a path (of empty cells) from any cell $u \in IO$ to $v$ (i.e., the load to be stored appears on $u$).

    \item Retrieval: A load at cell $v\in \storage$ can be \emph{retrieved} from \storage via a path (of empty cells) to any cell $u \in IO$.

    \item Relocation: A load can be \emph{relocated} within $\storage^+$ to an empty cell via a path (of empty cells).
\end{itemize}

Denote the arrival sequence, i.e., the order in which loads arrive so as to be stored, by ${\inc = (a_1, \ldots, a_{\nobjs})}$.
Without uncertainty, the departure sequence, i.e., the order in which loads are to be retrieved, is fixed to be ${\dep = (1, \ldots, \nobjs)}$ without loss of generality, as loads can always be relabeled.

We now introduce the notion of \emph{bounded perturbation} to model the uncertainty of arrival/departure sequences. 
\begin{definition}
Let $\tilde{S}$ be a permutation of sequence ${S=(s_1, \dots, s_n)}$.
We say that $\tilde{S}$ is a \emph{$k$-bounded perturbation} of $S$ if for every pair of elements $s_i, s_j$ with $i < j$ in the reversed order in $\tilde{S}$ (i.e., $s_j$ appears before $s_i$): ${j-i \le k}$.

\noindent \textbf{Example:} A 2-bounded perturbation of ${S_5 = (1,2,3,4,5)}$ is ${(2, 3, 1, 5, 4)}$.
In this example the inverted pairs are $(1,3)$, $(1,2)$, and $(4, 5)$.
The sequence $(1,3,5,2,4)$, however, is \emph{not} a 2-bounded perturbation of $S_5$ since it contains the inverted pair $(2,5)$ whose elements appear too far apart in $S_5$.
\end{definition}

\begin{problem}[\emph{Robust} Storage and Retrieval with Minimum Relocations (\rstormr)]
Given a storage area $\storage$ with $\rows$ rows and $\cols$ columns, arrival and departure sequences $\inc$ and $\dep$, and an integer $k > 0$, find a minimum-length sequence of actions that stores all loads according to $\inc$ and then retrieves them according to a sequence $\tilde{\dep}$, which is a $k$-bounded perturbation of $\dep$.
$\tilde{\dep}$ is revealed one load at a time during the retrieval phase and is not known during storage.
\end{problem}

\begin{remark}
Since we fix $\dep = [n] \coloneqq (1, \ldots, \nobjs)$, one may drop \dep from the input. Refer to $A$ as the input.%
\end{remark}

\noindent \textbf{Additional objectives.}
In addition to relocation actions, we consider two additional minimization objectives:
\emph{(i)} \textbf{I/O row usage}, defined as the number of actions in which a load is present on the I/O row at the start of the action.
This metric approximates the time during which the I/O row is occupied by a static load.
This time matters in practical storage systems as the I/O row may be needed for other transport operations besides providing access to \storage. 
\emph{(ii)} The \textbf{total distance} traveled by the loads throughout the sequence of actions. 

\section{Analysis of Solutions Without Relocations}
A natural goal is characterizing the conditions for solving \rstormr without relocations.
We show that $\Theta(k)$ columns are necessary and sufficient.
Interestingly, the bounds hold for any number of rows $r>1$, i.e., the number of columns is the key parameter that governs zero-relocation solutions under uncertainty. %

\subsection{Preliminary analysis}
We recall useful results for relocation-free solutions to \stormr (i.e., when $k=0$)~\cite{rss25}.

\begin{definition}
An arrangement \arr \emph{satisfies} an arrival (resp. departure) ordering \inc  (resp. \dep) if all loads can be stored (resp. retrieved) in the order specified by \inc (resp. \dep) with one action per load with \arr as the final (resp. initial) arrangement.
\end{definition}

\noindent View the problem from the following reverse perspective:

\begin{observation} \label{obs:reverse}
An arrangement \arr satisfies an arrival order \inc iff \arr satisfies the departure order where \inc is reversed.
\end{observation}

Following \Cref{obs:reverse}, it suffices to treat \stormr as the  problem of finding an arrangement \arr that satisfies two departure orders, the true departure order $[n]$ and a permutation $A^R$ of $[n]$, where $A^R$ is the reverse of $A$.
For example, given $A = (4,1,7,6,3,2,9,8,5)$ and $D = (1, 2, \ldots, 9)$ as in \Cref{fig:intro_ex}, the two departure orders to be satisfied are the original $D$ and $A^R = (5, 8, 9, 2, 3, 6, 7, 1, 4)$.

Deciding whether \arr satisfies a departure order $\dep$, amounts to checking these \emph{adjacency conditions}:

\begin{observation} \label{prop:local-orig}
    An arrangement \arr satisfies a departure order $\dep=(d_1, \ldots, d_{\nobjs})$ iff every load $d_i$ is either in the bottom row or is adjacent in \arr to a load $d_j$ that departs earlier, i.e., $j < i$.
\end{observation}

One can verify that the adjacency conditions hold for each of $\dep$
and  $\inc^R$ in the arrangements of \Cref{fig:intro_ex}, e.g., consider load 2 in (c); 2 is adjacent to 1 (as needed for $\dep$) and also adjacent to 8 (as needed for $\inc^R$). A key result for \stormr is an algorithm \emph{\prevalg} (baseline storage), that finds a zero-relocation solution given $c \ge 3$~\cite{rss25}:
\begin{theorem} \label{thm:3-col}
Let \storage  be a $\rows \times \cols$ storage area with $\cols \ge 3$ columns, and $\inc$ and $\dep$ be storage and retrieval sequences, respectively, for $n\le rc$ loads.
An arrangement \arr that satisfies both $\inc$ and $\dep$ can always be found in $O(n)$ time.
\end{theorem}

Robustness to perturbations, however, requires revisiting the adjacency conditions for \rstormr.

\begin{definition}
Given a departure/arrival sequence $S$, an arrangement \arr is \emph{$k$-robust for $S$} (or just \emph{$k$-robust} if $S$ is obvious) if \arr satisfies every $k$-bounded perturbation $\tilde{S}$ of $S$. 
\end{definition}
\begin{prop}[\textbf{(Robust adjacency conditions)}]
\label{prop:local-pert}
Let $\dep=(d_1,\dots,d_{\nobjs})$ be a departure sequence with $n=rc$ and let $k \ge 0$.  
An arrangement \arr is $k$-robust for $\dep$ if and only if every load $d_i$ is either in the bottom row or is adjacent in \arr to a load $d_j$ that appears in $\dep$ at least $k+1$ loads earlier in $\dep$, i.e., $i-j \ge k+1$.
\end{prop}

\begin{proof}
Suppose \(\arr\) is \(k\)-robust but for some \(i\) the load \(d_i\)
is neither on the bottom row nor adjacent to any
\(d_j\) with \(j\le i-(k+1)\).  
Define a $k$-bounded perturbation $\dep'$ of \dep by taking load $d_i$ and inserting it right after load $d_{i-(k+1)}$, or placing $d_i$ as the first load if $i< k+1$.
That is, move $d_i$ as early as possible and push other loads back.
Since \(d_i\) has no neighbor among
\(\{d_1,\dots,d_{i-(k+1)}\}\), by \Cref{prop:local-orig},
\(\arr\) doesn't satisfy \(\dep'\), which contradicts \(k\)-robustness.

Now assume that the adjacency condition holds in \(\arr\), and let
\(\dep'\) be any \(k\)-bounded perturbation of \(\dep\). 
We claim that \Cref{prop:local-orig} applies to
\(\dep'\). 
Indeed, for each \(i\), if \(d_i\) is not on the front row it is adjacent to some \(d_j\) with \(j\le i-(k+1)\). 
No load can appear more than \(k\) spots earlier, so any load originally before position \(i-(k+1)\) still appears before \(d_i\) in \(\dep'\). 
Thus, the adjacency condition of \Cref{prop:local-orig} holds for \(d_i\).
Since this is the case for every \(d_i\), \(\arr\) satisfies
\(\dep'\).
As \(\dep'\) is an arbitrary perturbation, \(\arr\) is \(k\)-robust.
\end{proof}

\begin{figure}[tb]    
\centering
\includegraphics[width=.85\columnwidth]
{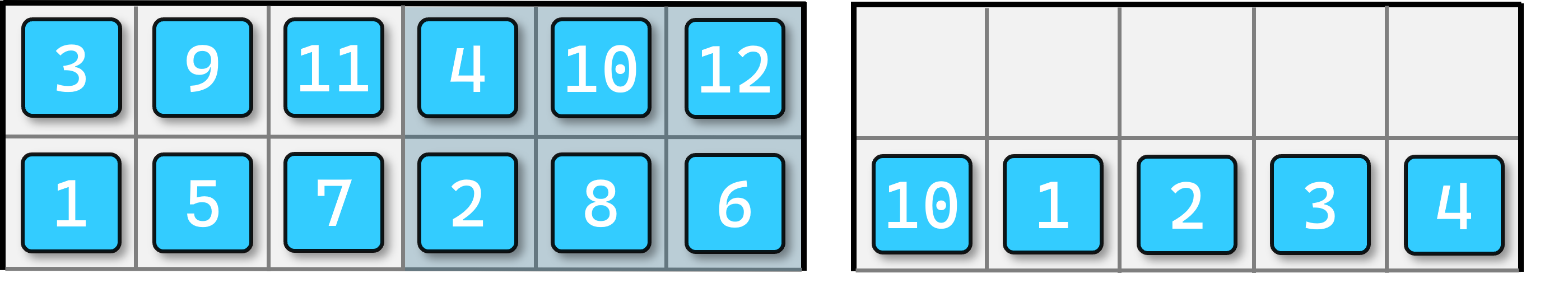}
\caption{Examples of the column bounds for $k = 1$.
\textbf{Upper bound} (left):
Consider a $2 \times 6$ grid with $\inc = (7, 3, 11, 1, 9, 4, 6, 12, 2, 10, 8, 5)$. Per \Cref{thm:UB}, we partition $\inc$ into subsequences of odd and even loads: $A_0 = (7, 3, 11, 1, 9, 5)$ and $A_1 = (4, 6, 12, 2, 10, 8)$.  
We treat $A_0$ as a \stormr instance on a $2 \times 3$ grid with $D = (1, 3, 5, 7, 9, 11)$ and use the solution to fill the leftmost 3 columns. Similarly, we treat $A_1$ as a separate instance for the rightmost 3 columns (shaded). The combined solution is a robust arrangement.
\textbf{Lower bound} (right):
Consider $A^R$ starting with 10, 3, 4. As $k = 1$, we must have 1 and 2 on the bottom row.
Treating $A^R$ as a departure sequence to satisfy, 10 must also be on the bottom row.  
Next, to avoid placing 3 in the bottom row, we must store it in a cell adjacent to both 10 and 1. As no such cell exists, 3 must also be placed in the bottom row.  
Similarly, 4 must be adjacent to one of 1 and 2 and one of 10 and 3 to avoid the bottom row. Again, this is not possible, so 4 is also assigned to the bottom row. Thus, 5 columns are required.
}
\label{fig:bounds}
\end{figure}

\noindent Putting everything together, we have the following:
\begin{corollary} \label{cor:k-robust-for-zero-reloc}
The problem of finding a zero-relocation solution for \rstormr is equivalent to finding an arrangement that satisfies $\inc$ and is $k$-robust for \dep.
\end{corollary}

Given the equivalence we ask: how many columns in terms of $k$, guarantee the existence of this arrangement?

\subsection{Bounds on required number of columns}
The bounds are tight up to 1.5, for the necessary and sufficient number of columns of a zero-relocation \rstormr solution.

\begin{theorem}[(Upper bound)]
\label{thm:UB}
For an $\rows \times \cols$ storage area \storage with any $\rows \ge 1$ rows,
$3k+3$ columns suffice to guarantee a zero-relocation solution to \rstormr.
\end{theorem}
\begin{proof}
We present an algorithm and prove its correctness.

\textbf{Algorithm.} 
The idea is to partition the loads into $k+1$ subsets, each of which will be treated as an independent \stormr instance:
More specifically, we partition the arrival sequence \inc into $k+1$ subsequences where each subsequence contains $a_i$'s of the same congruence class (i.e., remainder) modulo $k+1$.
Formally, for each $j \in \{0, \ldots, k\}$, define ${A_j \coloneq (a_i \mid a_i \equiv j \pmod{k+1})}$.
These subsequences form a partition of the loads.
We then assign the loads in each $A_j$ to 3 dedicated contiguous columns, thus using $3k+3$ columns in total.
We then treat each $A_j$ as a \stormr instance for 3 columns and apply algorithm \prevalg to assign the loads.
See \Cref{fig:bounds}.

\textbf{Correctness.}  %
\Cref{thm:3-col} guarantees that (non-robust) adjacencies are met for each $A_j$ within the respective 3 columns.
We claim that the adjacencies are also robust, as required by \Cref{prop:local-pert}.
Fix a load $x$ that is not assigned to the bottom row. Load $x$ has a neighboring load $y$ in $A_j$ that departs earlier.
As $A_j$ consists of loads with departure indices congruent modulo $k+1$, $y$ must appear at least $k+1$ positions before $x$ in \dep, as required.
\end{proof}

\begin{theorem}[(Lower bound)]
\label{thm:LB}
For an $\rows \times \cols$ storage area \storage with $\rows > 1$ rows,
$2k+3$ columns are necessary to guarantee a zero-relocation solution for \rstormr.
\end{theorem}

\begin{proof}
We construct an instance with $n \ge 2k+3$ loads that completely fill the storage area.
See example in \Cref{fig:bounds} (right).
We set the first $k+2$ loads of $A^R$ to be $n, {k+2}, {k+3}, \ldots, 2k+2$ (i.e., these are the last $k+2$ loads in $A$); the rest of the sequence is arbitrary. Set $D= (1,2,\ldots,n)$.
First, notice that the loads $1,2,\ldots,k+1$ must all be placed in the bottom row of \storage, as each of them may need to depart first under a perturbation.

We now prove by induction that the first $k+2$ loads on $A^R$ must also be placed in the bottom row, in addition to the loads $1,2,\ldots,k+1$, which would require $2k+3$ columns in total.
For the base case, since $n$ is first in $A^R$ it must be placed in the bottom row.
Let us assume that the loads $n, k+2, \ldots, k+i$ must be placed in the bottom row (for some $2 \le i < k+2$) and show that the same holds for load $k+i+1$.
More specifically, we argue that there is no way to satisfy the required local adjacencies for $k+i+1$ otherwise:
To avoid placing load $k+i+1$ on the bottom row, it must be adjacent to one of the loads $1,2,\ldots,i$ (to meet adjacency requirements for $\dep$) and also be adjacent to one of the loads appearing before it $A^R$, namely, $n, k+2, \ldots, k+i$.
However, the loads $1,2,\ldots,i$ are in the bottom row and so are the loads $n, k+2, \ldots, k+i$, by the induction hypothesis. There is clearly no cell above the bottom row that is adjacent to two different cells on the bottom row.
Therefore, we must place $k+i+1$ in the bottom row as well, as required.

Finally, notice that the analysis already holds for a storage area with two rows and is the same for a higher number of rows. %
\end{proof}

Rephrasing our bounds, we conclude that for ${k/c \approx k/(c-3) \le 1/3}$ we can avoid relocations, whereas for ${k/c \approx k/(c-3) \ge 1/2}$ we cannot.
We observe that even though the lower bound requires $2k+3$ columns in general, fewer columns may suffice for some instances, as occurs in \Cref{fig:intro_ex} (where 3 columns suffice for $k=1$ instead of 5).

With guidance from the theoretical bounds, we now turn to a practical algorithm, which achieves a high success rate at finding robust arrangements even when $k$ is roughly half the grid width, closely matching \Cref{thm:LB}.

\begin{figure}[htb]
    \centering
    \begin{overpic}[width=.86\columnwidth]{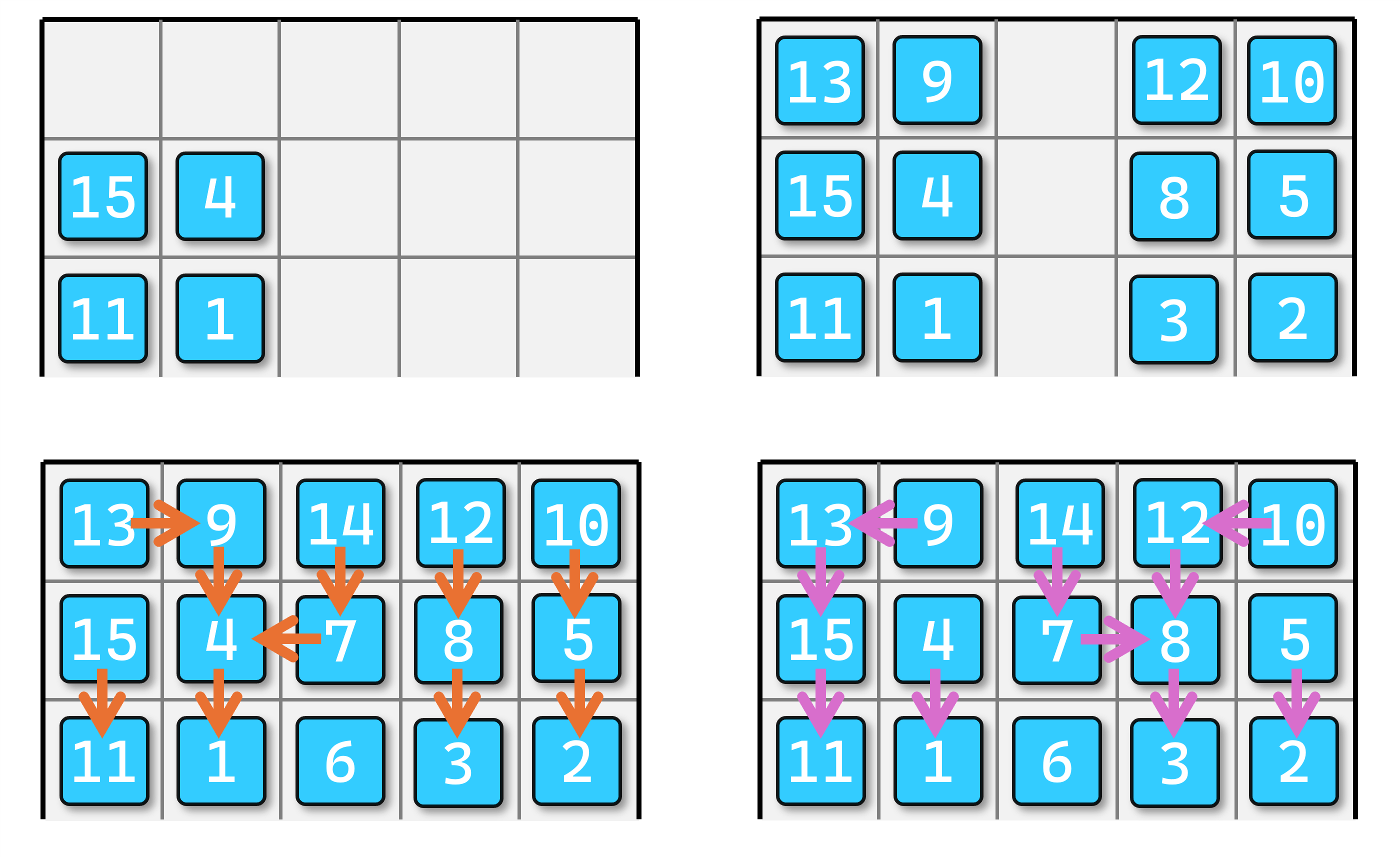}
    \put(22, 30.8){{\small $(a)$}}
    \put(73, 30.8){{\small $(b)$}}
    \put(22, -1){{\small $(c)$}}
    \put(73, -1){{\small $(d)$}}
    \end{overpic}
    \caption{The storage algorithm for a ${3 \times 5}$ grid with $A^R = (11, 3, 15, 8, 2, 7, 13, 12, 1, 9, 10, 14, 6, 4, 5)$, $D=[15]$, and $k=2$.
    (a)(b) Snapshots showing partial arrangements as the algorithm runs.
    First column pair:
    1 and 11 are placed per column initialization.
    Next, we set $x=4$ as the first load in $D$ that can be placed above 1. In the inner loop, $y = 3$ is discarded as it does not satisfy adjacencies for $D$; we proceed to $y = 15$, which is valid.
    The main loop then sets $x = 7$, but no matching valid $y$ is found until we reach $y = 7$, at which point 7 cannot be placed in $R$.
    Continue to $x = 8$, reaching the same conclusion.
    Then, for $x = 9$, a valid match is found with $y = 13$.
    (c)(d) Arrows showing valid adjacencies for departures and arrivals (treating arrivals in reverse).%
    }
    \label{fig:storage_ex}
\end{figure}
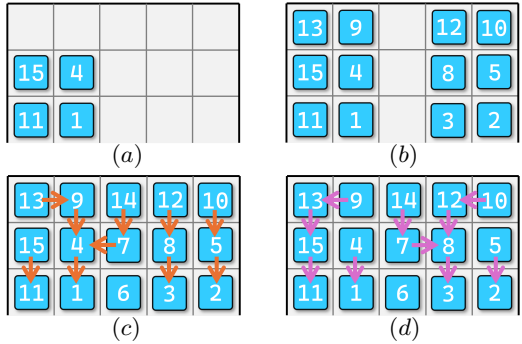

\section{Finding a Robust Arrangement}
The previous section provides guidelines on what uncertainty levels can be supported by a given storage area. We now constructively find a robust arrangement without relocations under these conditions.
Although $2k+3$ columns are required in general, fewer columns might suffice for a given instance.
In practice, instead of fixing $k$, one might rather adapt it based on the storage and retrieval sequences to maximize robustness by maximizing $k$.

This raises the following problem: given a \rstormr instance, find an arrangement \arr that is $k$-robust for $\dep=[n]$ and also satisfies $\inc$.
Such an \arr corresponds to a zero-relocation solution as we establish in \Cref{cor:k-robust-for-zero-reloc}.
We call such an arrangement \emph{valid}.
Similarly, a load $x$ is \emph{valid} if it is on the bottom row or is both \emph{(i)} \emph{$D$-valid}, i.e., adjacent to load $y$ with $y\le x-k-1$ and \emph{(ii)} \emph{$A$-valid}, i.e., adjacent to a load $z$ that appears before $x$ in $A^R$.

\subsection{Main algorithm}

The main idea is to fill pairs of adjacent columns, denote them as $L$ and $R$ (for left and right), from the bottom row upwards.
Jointly iterate over $\dep$ and $A^R$ and greedily find pairs of loads that we can store adjacent to each other on $L$ and $R$ so that the loads are valid in the current (partial) arrangement. In this setup, we may skip loads as we iterate until we find valid loads to assign.
We aim to have column $L$ maintain bottom-up adjacencies for $A^R$ while column $R$ maintains the same for $D$.
For the loads on $L$ to also satisfy the adjacencies for $D$ we rely on horizontal adjacencies. Analogously, for $R$ and $A^R$.
If the algorithm reaches a stage where there is no pair of loads that can be assigned, a failure is declared. Refer to \Cref{alg:storage} (where an even number of columns is assumed for simplicity) and to \Cref{fig:storage_ex}.

\textbf{Initializing a column pair.}
For each column pair $(L,R)$, the bottom cells are initialized to establish vertical adjacency chains (lines \ref{l:botL}-\ref{l:botR}).
For $L$, we initialize with next unassigned load in $A^R$.
We aim to place the first $k+1$ loads from $D$ ($1, \ldots, k+1$) in the front row, as these loads may need to depart first under perturbations. %
Therefore, we initialize $R$ with the smallest unassigned load.

\textbf{Main loop (\cref{l:loop}).}
Find the next load $x$ that can be assigned to $R$ by taking $x$ to be the first unassigned $D$-valid load, i.e., ${x \ge x' + k + 1}$, where $x'$ is assigned and adjacent to the lowest empty cell on $R$.  
With $x$ fixed, iterate over $A^R$, until a matching load $y$ is found so that all adjacency conditions are met. %
When checking adjacencies, consider all neighboring cells of $x$ and $y$ that have a load assigned, to maximize success (i.e., $y$ might "rely" on a load to its left rather than on $x$ or the load below $y$).
If $y = x$, \textbf{continue} as we aim to assign a pair of loads at each step.\footnote{As a speed up, when not filling the last column pair, \textbf{break} from the inner loop instead to examine the next load for $x$.
This is done because $y$ must appear before $x$ in $A^R$ for $x$ to satisfy the adjacency for $A^R$.
If $y = x$, there is no point considering later loads in $A^R$.}
If the main loop iterates without finding a valid pair to assign, return failure.

\textbf{Assigning the last loads.} To increase the likelihood of the assignments succeeding for last loads assigned, we leave special column(s) empty for them, to be filled after all other columns.
If $\cols$ is odd, we leave $C_3$ (the third column) as the last empty column, with the remaining loads assigned in the order in which they appear in $\dep$.
Otherwise, we use $C_3$ and $C_4$ as the last column pair to be filled.
By leaving $C_3$ and possibly $C_4$ as the last columns, we use $C_2$ and $C_4$ (or $C_5$), which contain early appearing loads in \dep and $\inc^R$, for potential horizontal adjacencies.

Since we always check whether loads are valid before assigning them, any returned arrangement is valid.
However, whenever a failed is returned, it does not mean that there is no valid arrangement. Therefore, we present an enhancement that considers more potential arrangements.

\begin{algorithm}[t]
\caption{Find robust arrangement}
\begin{small}
\label{alg:storage}
\KwIn{\rstormr instance  $\rows , \cols \ge 5, \inc , \dep, k$}
\KwOut{A valid arrangement \arr or \textbf{Failure}}

 $\mathcal{P} \gets [(C_1,C_2), (C_5, C_6), \dots, (C_{c-1},C_{c}), (C_3, C_4)]$\;

\ForEach{pair $(L, R)$ in $\mathcal{P}$}{
    $X \gets$ iterator over unassigned loads in $\dep$ starting from $k+2$\; %
    $Y \gets$ iterator over unassigned loads in $A^R$\;

    bottom cell of $L$ $\gets$ first unassigned load from $Y$\; \label{l:botL}
    bottom cell of $R$ $\gets$ smallest unassigned load\; \label{l:botR} %

    \While{$L$ and $R$ not full} 
    {  
        Advance $x \in X$\;
        \lIf{$X$ is exhausted}{\Return \textbf{Failure}}
        \lIf {$x$ is not $D$-valid when assigned to $R$}
        { 
            \textbf{continue}
        }

        \For{$y$ in $Y$}{ %
            \lIf{$y = x$}{
                \textbf{continue} %
            }
            \If{$x$ and $y$ are valid when $x$ and $y$ assigned to $R, L$ respectively}{ 
                Assign $x, y$ to $R, L$ respectively\;
                \textbf{break}\;
            }
        }
    } \label{l:loop}
}

\Return resulting arrangement \arr\;
\end{small}
\end{algorithm}

\subsection{Load-skipping enhancement} %
We introduce an enhancement to \Cref{alg:storage} to increase its success rate.
A valid arrangement only requires that the first load in $A^R$ is in the bottom row, while freedom exists for other loads.
We exploit this by considering all options for the choice of the first load of $A^R$ assigned to the left column.
That is, for the first column pair filled, we set the iterator $Y$ to start at an offset, initializing the arrival adjacency chain from the middle of $A^R$.
For the remaining column pairs, we proceed as normal. 
The enhanced version tries all possible offsets from 0 to $n-r$, which one can run in parallel.

\section{Solving the 2D Grid Relocation Problem}
\rstormr may require relocations as $k$ increases. This section addresses the retrieval phase where, given an initial arrangement \arr, the goal is to minimize relocations during retrieval.
Already for stack-based storage and a known retrieval order, the problem is NP-hard~\cite{BRP-hard}. Our setting is more involved if a load to be retrieved is blocked by other loads as (i) there are multiple retrieval paths (that determine the loads to relocate), and (ii) there are more options for where to relocate. %

Given the above and so as to minimize I/O row usage, we impose the following constraints:
The I/O row must be empty after each load is retrieved, i.e., all relocated loads must end up in \storage.
Furthermore, when loads on the I/O row are stored back in \storage, simply return them to their original cells in $\storage$, as in \Cref{fig:intro_ex} (d)-(f).
This choice reflects that we choose \arr to enable robust relocation-free retrieval according to \dep.

\textbf{Relocation procedure.}
Let $x$ denote the \emph{target load} to be retrieved. Compute a retrieval path $\pi$ from $x$ to the I/O row that minimizes blocking loads, breaking ties by path length. If $\pi$ does not pass through other loads, $x$ can be retrieved.
Otherwise, the blocking loads along $\pi$, i.e., the \emph{blockers}, must be relocated, starting from the outermost and proceeding inward. Relocate each blocker $b$ in a greedy manner, assigning it to a favorable empty cell that is not on $\pi$ (keeping $\pi$ clear for $x$ to be retrieved).
See \Cref{fig:multiblocker_ex} for an example.

First consider the case where $b$ has reachable empty cells in \storage.
In this case, aim to assign $b$ to a destination cell from which $b$ will not be relocated again during a subsequent retrieval.
To this end, compute a set $U$ of \emph{unblocked} loads, which are loads that have direct access to the I/O row.
For each candidate assignment of $b$ to a destination cell, check whether a load in $U$ becomes blocked.
A load $y$ is considered blocked, if it departs before all of its adjacent loads as well as $b$, per \dep, and has no direct path to the I/O row.
If there is a destination cell that does not make a load of $U$ blocked, assign $b$ to it, preferring the closest such cell.
Otherwise, assign $b$ to the closest empty cell.
In both cases, ensure that we assign $b$ to a cell that leaves sufficiently many reachable empty cells in \storage for the remaining blockers still on $\pi$. %

Alternatively, there are no empty cells in \storage for relocations and the remaining blockers are relocated to the I/O row.
These blockers will be relocated back to \storage after $x$ is retrieved. All relocations in this case use $\pi$ to go to/from the I/O row. Assuming $\rows \le \cols$, the number of cells on the I/O row suffices.

\begin{figure}[t!]
    \centering
    \begin{overpic}[width=1\columnwidth]{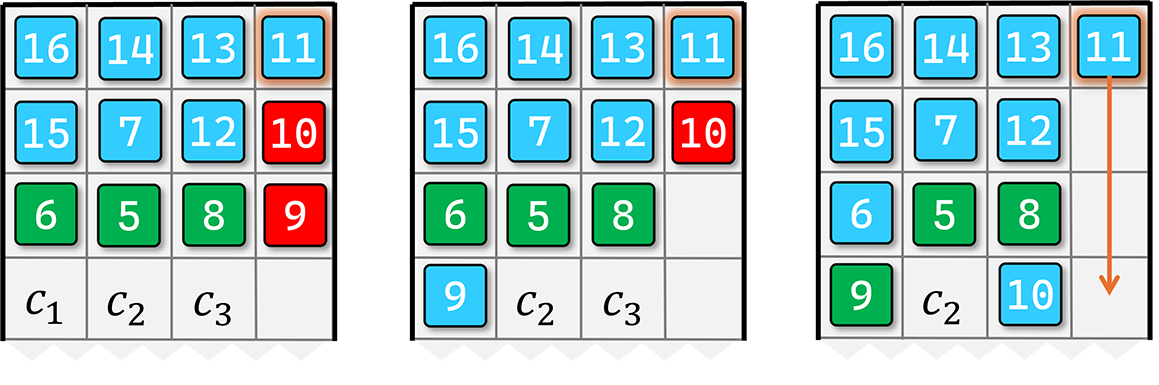}
    \put(13, -1){{\small $(a)$}}
    \put(47, -1){{\small $(b)$}}
    \put(81, -1){{\small $(c)$}}
    \end{overpic}
    \caption{Example relocations: (a) Load $11$ is to be retrieved. $\pi$ is a straight downward path and $9$, $10$ are blockers (red). To relocate $9$, we compute the set of unblocked loads $U = \{5,6,8\}$ (green). Destination cell $c_3$ is not chosen because it would disconnect $10$ from empty cells and $c_2$ is also not chosen as placing $9$ there would block $5$. (b) $9$ is relocated to $c_1$ which keeps $6$ unblocked due to its adjacency to $5$. (c) $U$ is recomputed and $10$ is relocated to $c_3$, as it does not block any loads.
    }
    \label{fig:multiblocker_ex}
\end{figure}

\section{Experimental Evaluation} \label{sec:exp}
The objective of the experimental evaluation is to measure the improvement achieved by the proposed storage and retrieval strategies over baselines for total number of relocations,
I/O-row usage, %
and total load distance traveled in square grids.
Since the theoretical analysis establishes the ratio $k/c$ as a key parameter affecting relocations in \rstormr, four values of the ratio are considered, i.e., $k$ values of $0.25c, 0.5c, 0.75c, c$ per grid size with $c$ columns.
We run $50$ trials for each combination of grid side length and \( k \), always at 100\% density.
Each trial randomly generates $\inc$,  which is given to the storage algorithm, and then draws a $k$-bounded perturbation of $(1, 2, \ldots, n)$, which is revealed one at a time during retrieval.

\begin{figure*}[t]
    \centering
    \begin{subfigure}{0.563\columnwidth}
        \centering
        {\small{\hspace{2.4em}BaseS BaseR}}\\[0.15em]
        \includegraphics[width=\textwidth]{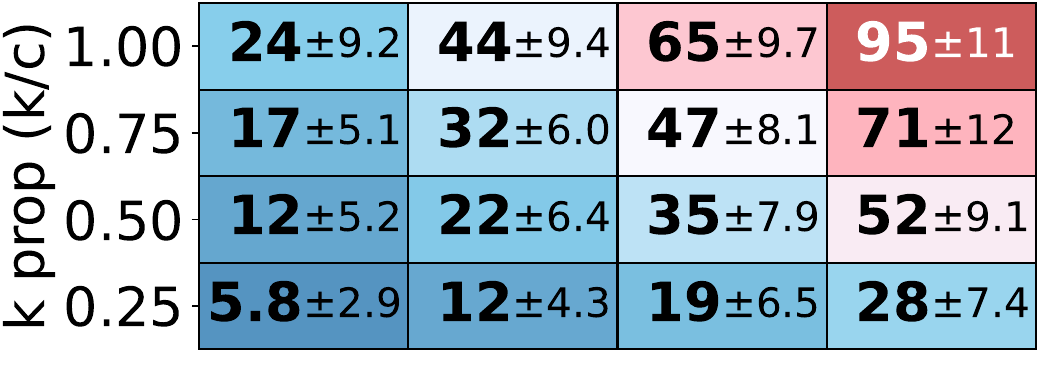}
    \end{subfigure}
    \begin{subfigure}{0.46\columnwidth}
        \centering
        {\small{\imps BaseR}}\\[0.15em]
        \includegraphics[width=\textwidth]{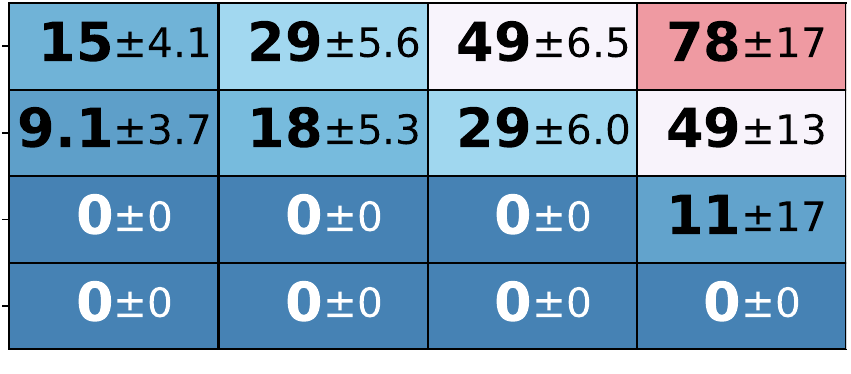}
    \end{subfigure}
    \begin{subfigure}{0.46\columnwidth}
        \centering
        {\small{BaseS \textbf{ImpR}}}\\
        \includegraphics[width=\textwidth]{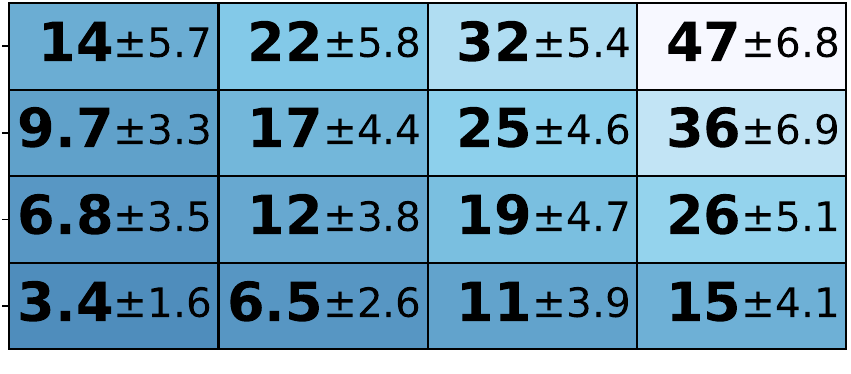}
    \end{subfigure}
    \begin{subfigure}{0.581\columnwidth}
        \centering
        {\small{\hspace{-3.1em}\textbf{\imps ImpR}}}\\
        \includegraphics[width=\textwidth]{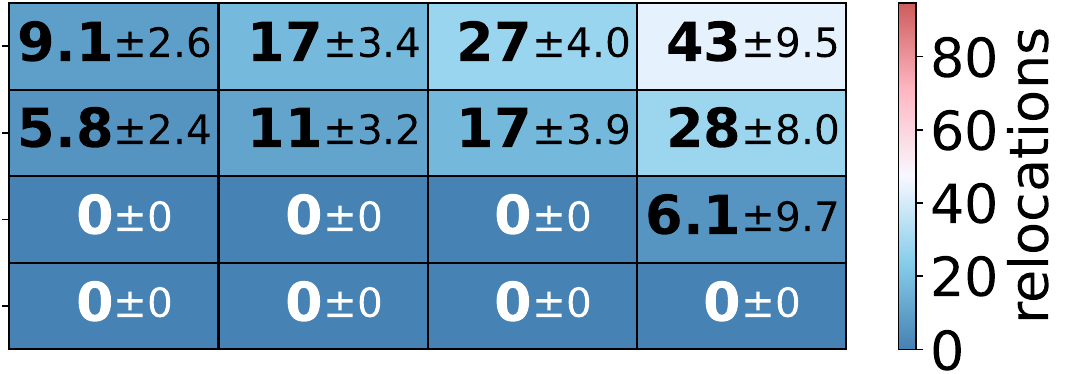}
    \end{subfigure}

    \begin{subfigure}{0.563\columnwidth}
        \centering
        \includegraphics[width=\textwidth]{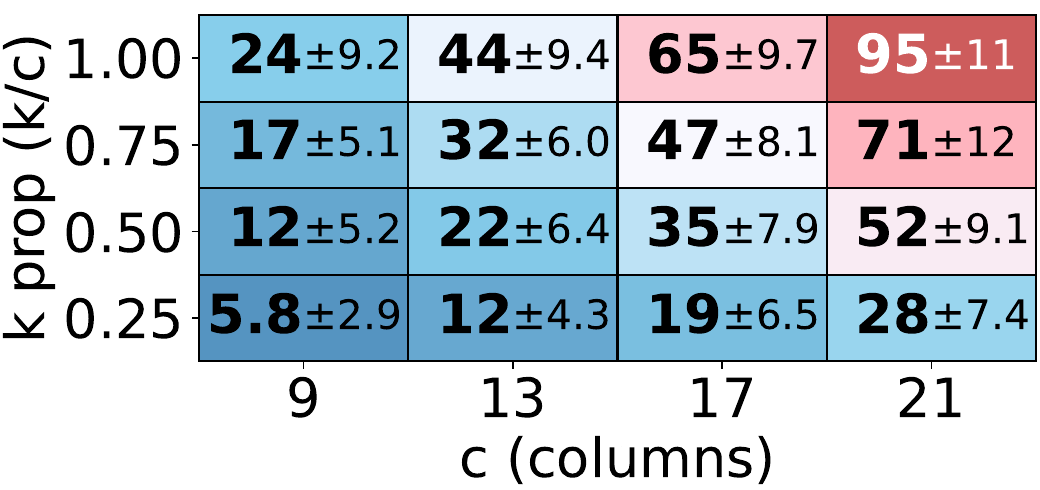}
    \end{subfigure}
    \begin{subfigure}{0.46\columnwidth}
        \centering
        \includegraphics[width=\textwidth]{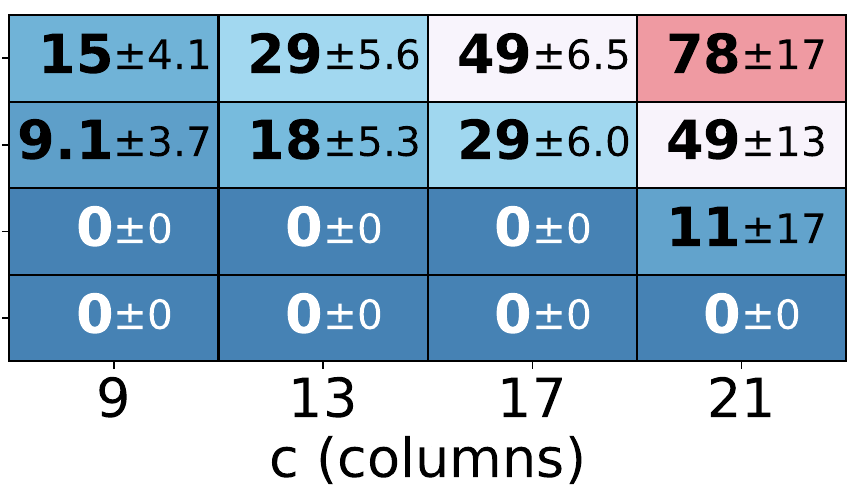}
    \end{subfigure}
    \begin{subfigure}{0.46\columnwidth}
        \centering
        \includegraphics[width=\textwidth]{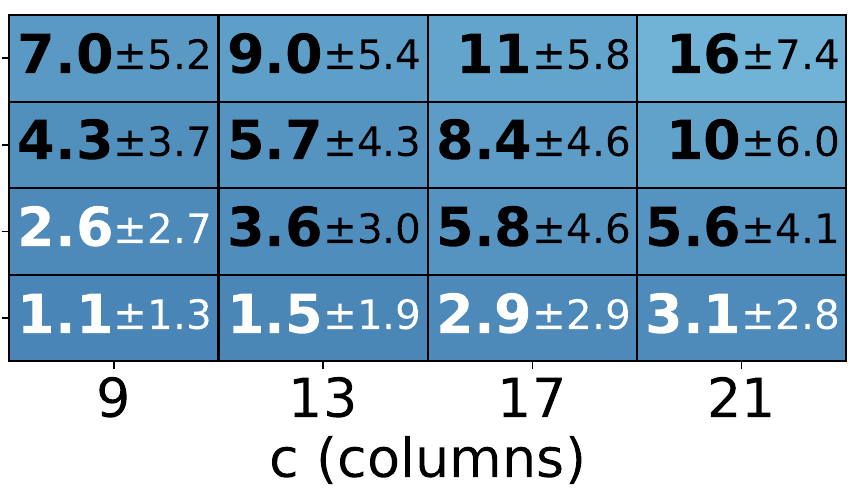}
    \end{subfigure}
    \begin{subfigure}{0.581\columnwidth}
        \centering
        \includegraphics[width=\textwidth]{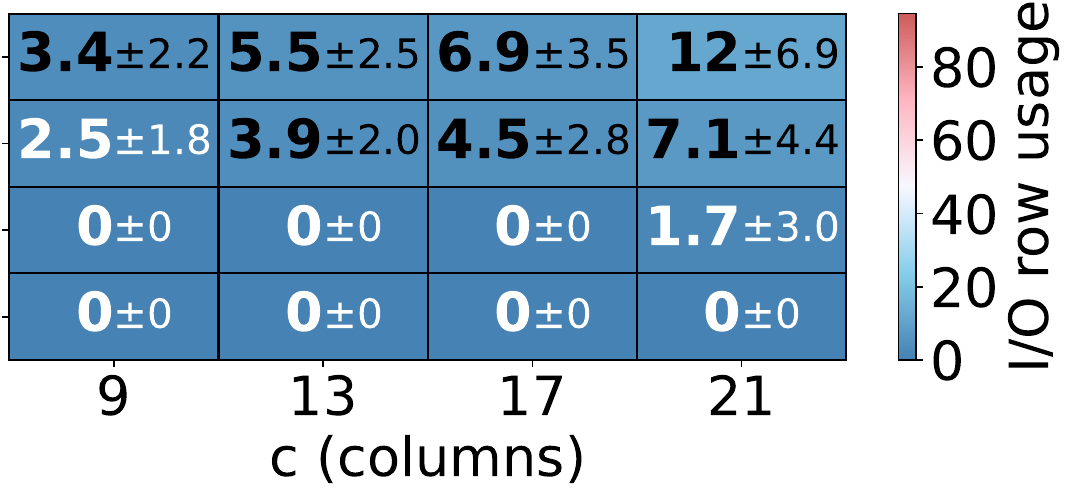}
    \end{subfigure}
    
    \caption{Mean relocations $\pm$ st. dev. (top) and mean I/O row usage $\pm$ st. dev. (bottom) for varying grids and $k$ values, comparing storage and retrieval algorithms.
    The heatmaps are the same in the two leftmost columns since BaseR always uses the I/O row.
    }
    \label{fig:exp:heatmap}
\end{figure*}

{
\captionsetup[subfigure]{aboveskip=0pt, belowskip=0.5pt}
\begin{figure}[!t]
    \centering
    \includegraphics[width=0.47\textwidth]{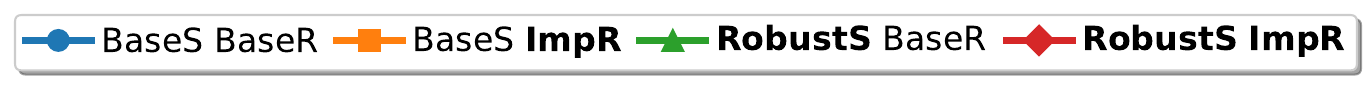}

    \begin{subfigure}{0.23\textwidth}
        \centering
        \includegraphics[width=0.95\textwidth]{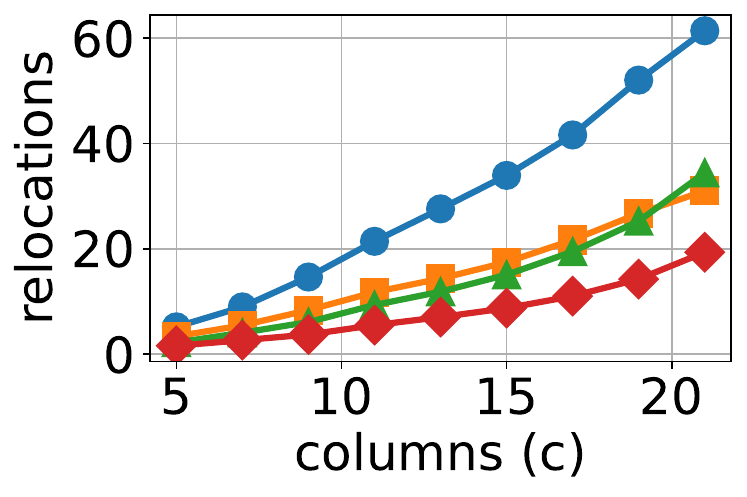}
        \caption{Relocations vs. grid size.}
    \end{subfigure}
    \hfill
    \begin{subfigure}{0.23\textwidth}
        \centering
        \includegraphics[width=0.95\textwidth]{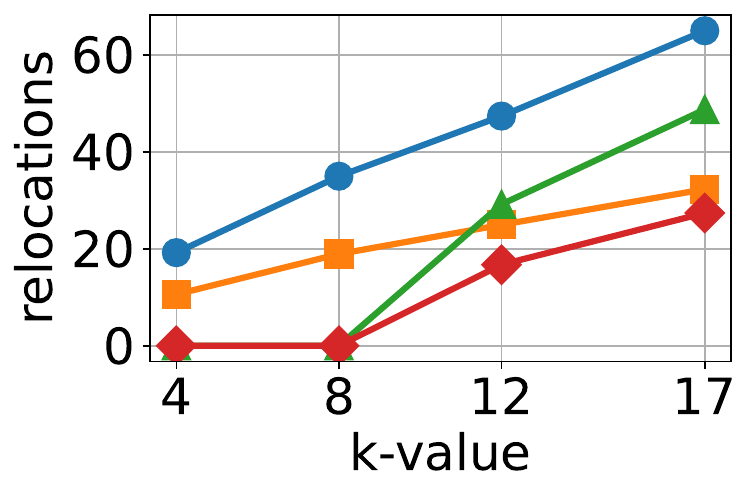}
        \caption{Relocations vs. $k$.}
    \end{subfigure}

    \begin{subfigure}{0.23\textwidth}
        \centering
        \includegraphics[width=0.95\textwidth]{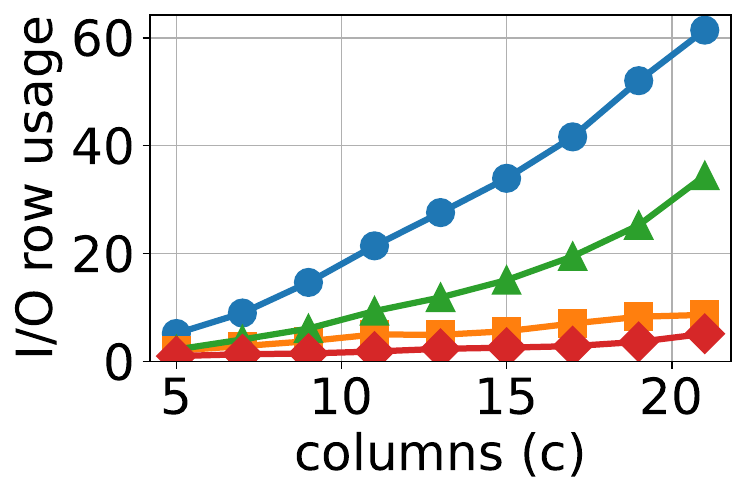}
        \caption{I/O row usage vs. grid size.}
    \end{subfigure}
    \hfill
    \begin{subfigure}{0.23\textwidth}
        \centering
        \includegraphics[width=0.95\textwidth]{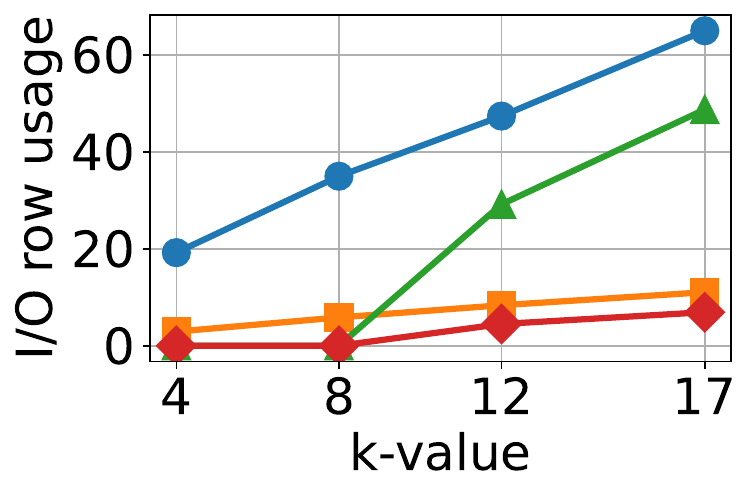}
        \caption{I/O row usage vs. $k$.}
    \end{subfigure}

    \begin{subfigure}{0.23\textwidth}
        \centering
        \includegraphics[width=\textwidth]{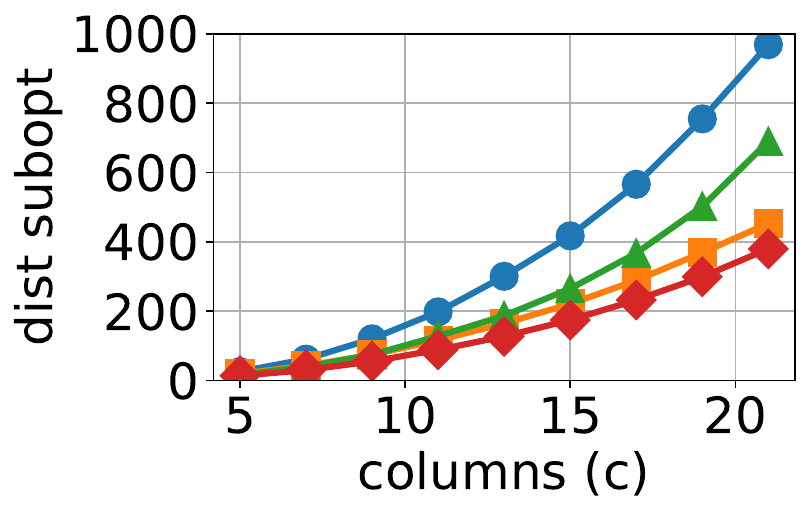}
        \caption{Distance subopt. vs. grid size.}
    \end{subfigure}
    \hfill
    \begin{subfigure}{0.23\textwidth}
        \centering
{\includegraphics[width=.97\textwidth]{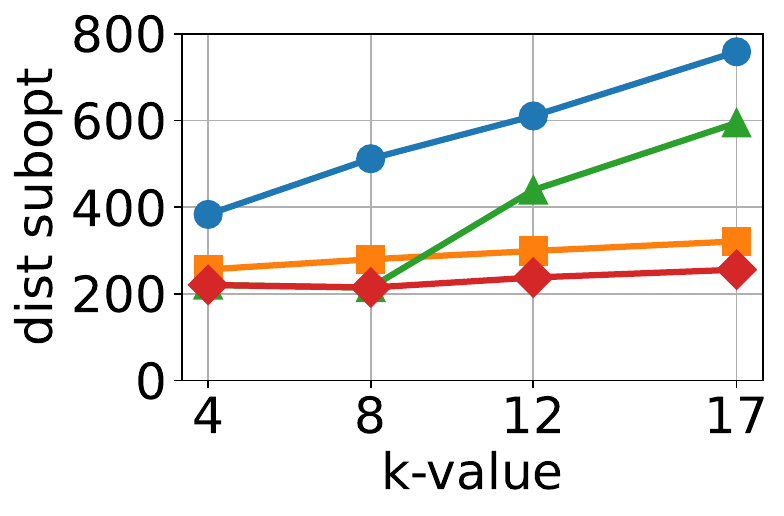}}
        \caption{Distance subopt. vs. $k$.}
    \end{subfigure}

    \caption{Relocations (top), I/O row usage (middle), and distance subopt. (bottom) for varying grid sizes and $k$ values. 
    In (a)(c)(e) we average across all four $k$’s per grid size. 
    In (b)(d)(f) we fix a $17\times17$ grid. 
    Each point averages 50 trials.}
    \label{fig:exp:reloc-io-plot}
\end{figure}
}

\begin{figure}[!h]
    \centering
    \begin{subfigure}[t]{0.23\textwidth}
        \centering
        \includegraphics[width=.95\textwidth]{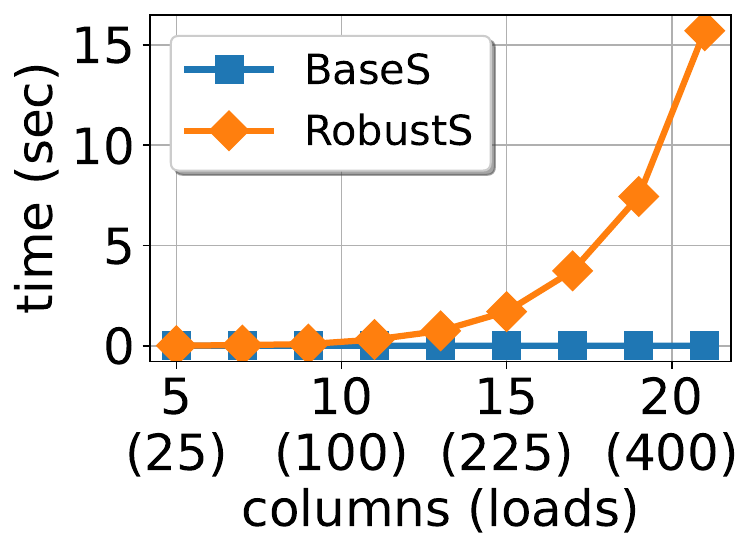}
        \caption{Storage running times.}
    \end{subfigure}
    \hfill
    \begin{subfigure}[t]{0.23\textwidth}
        \centering
        \includegraphics[width=.95\textwidth]{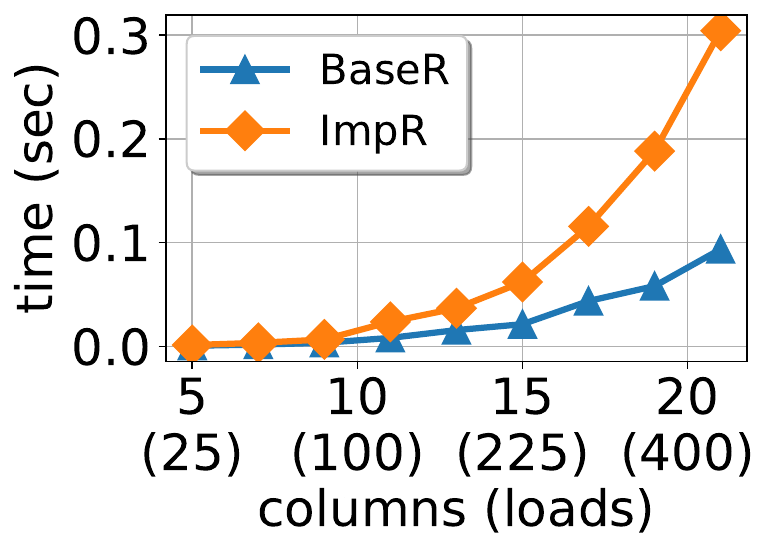}
        \caption{Retrieval running times.}
    \end{subfigure}
    \caption{Run-times averaged across $k$ values per grid size.}
    \label{fig:main:runtime}
\end{figure}

We compare the following algorithm variants for storage:
\begin{itemize}[leftmargin=3mm]
    \item \prevalg: store for 0 relocations assuming a known departure sequence per prior work \cite{rss25}.
    
    \item \imps: The proposed improved algorithm for a $k$-robust arrangement. If this algorithm does not find a $k$-robust arrangement (not always guaranteed to exist), decrement $k$ until one is found. For $k=0$, BaseS is used as a fallback, though this was not frequently observed in experiments.
\end{itemize}

\noindent We compare the following algorithm variants for retrieval: 
\begin{itemize}[leftmargin=3mm]
    \item BaseR:
    We find a retrieval path $\pi$ that contains the fewest blocking loads, breaking ties by path length.  When the target load $x$ is blocked, relocate the blocking loads to the I/O row, retrieve $x$, and place the blocking loads back in their original cells.
    Loads can be moved using $\pi$.

    \item \textbf{ImpR}: Our improved relocation algorithm.
\end{itemize}

\subsection{Results}
We present results using Python on an Apple M3 with macOS 15.5 for combinations of the above variants for storage and retrieval.
Numerical results for the number of relocations and I/O row usage are presented in \Cref{fig:exp:heatmap}. We also plot average results across all $k$ values for varying grid sizes, and separately across $k$ for a fixed grid size in \Cref{fig:exp:reloc-io-plot}. This figure includes \emph{distance suboptimality (subopt)} plots, defined as total grid path length traveled by loads minus a lower bound accounting for a fully packed grid, given by \(cr(r+1)\), i.e., excess over the optimum.

Experiments show that the combined (\imps + \impr) outperform the baseline and variants where only the storage or retrieval is improved. When $k$ is at most half the grid width ($k/c \le 0.5$), \imps nearly eliminates rearrangements.
These results are in line with \Cref{thm:LB}, which indicates that relocations are unavoidable when $k$ approaches half the grid width.
For larger $k$, relocations are reduced by up to 60-70\%.
For every grid size, \imps and \impr run in under 1 min. and 1 sec., respectively; see \Cref{fig:main:runtime}.

In nearly all cases, loads relocated within \storage are not relocated again, indicating that we avoid cascading relocations.
Qualitatively, relocation choices are typically straightforward: early on, blockers are moved to the I/O row due to limited space; later, as the grid empties, accessible cells simplify blocker placement decisions. While \impr significantly reduces relocations to the I/O row, \imps provides a further reduction, up to their elimination for $k/c\le0.5$.

Lastly, we show the improvement of \imps due to the enhancement; see \Cref{fig:exp:imps-abla}.

\begin{figure}[t]
    \centering
    \includegraphics[width=0.2\textwidth]{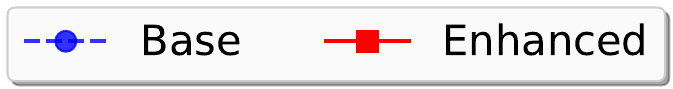}

    \begin{subfigure}[t]{0.233\textwidth}
        \centering
        \includegraphics[width=.95\textwidth]{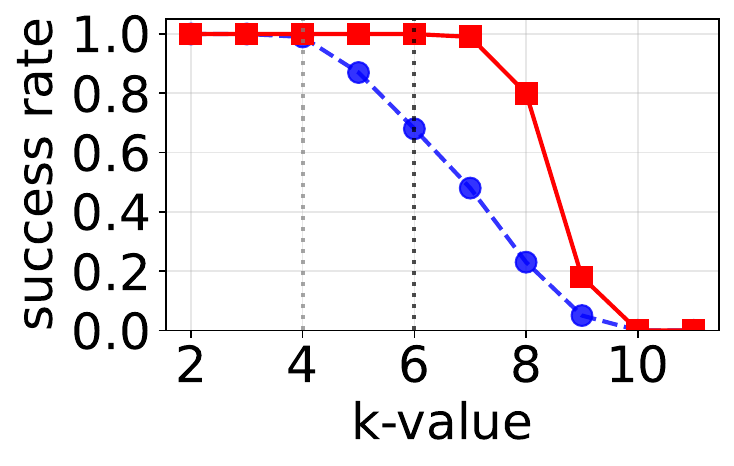}
    \end{subfigure}
    \hfill
    \begin{subfigure}[t]{0.233\textwidth}
        \centering
        \includegraphics[width=.95\textwidth]{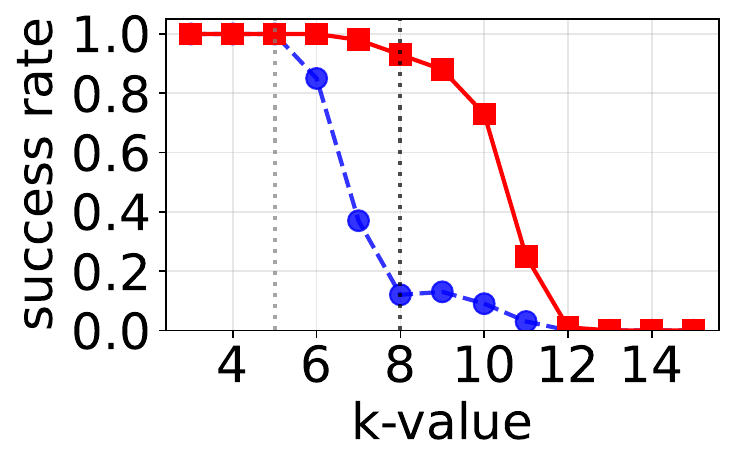}
    \end{subfigure}

    \caption{Ablation showing the success rate of \imps with and without the load-skipping enhancement on a $15 \times 15$ (left) and $19\times 19$ (right) grids (100 trials per data point). The enhanced \imps achieves 80\%+ success even for higher $k$ than the theoretical limits of 6 and 8 respectively (thick vertical dashed line) beyond which relocations can occur.
    }
    \label{fig:exp:imps-abla}
\end{figure}

\section{Conclusion}

This study expands the applicability of grid-based storage and retrieval by introducing uncertainty to this setting through novel theoretical and empirical results. It establishes design guidelines by relating the grid’s opening width to the feasibility of robust zero-relocation solutions. Empirically, our storage and retrieval approach significantly reduces relocations for storage at full capacity and confines most relocations to within the grid.

This work opens directions for further investigation: Is it possible to determine in polynomial time whether a robust arrangement exists for a given $k$ (strengthening the heuristic approach)? Can one always find a robust arrangement with $2k+3$ columns, closing the gap between our bounds? One could also consider the multi-robot case and MAPF reasoning in our setting.

\section*{Acknowledgements}\label{sec:ack}
We thank the reviewers and editorial staff for their insightful suggestions.
We thank Gur Lifshitz for useful discussions.
This work is supported in part by NSF awards IIS-1845888, IIS-2021628, IIS-2132972, CCF-2309866, and an Amazon Research Award.

%% file: related.tex
\section{Related Work}  \label{sec:related}
Many works study either high-density grid-based storage or ordered storage/retrieval problems involving a sequence of loads. Their intersection, however, has received less attention.
The puzzle-based storage (PBS) model addresses layout and retrieval for 2D grids with only one or few empty \emph{escort} cells that enable motion~\cite{gue2007puzzle, kota2015retrieval, bukchin2022comprehensive}.
Retrieval problems in PBS focus on one or a few target loads at time, without considering a complete retrieval sequence.
In settings where multiple loads are to be retrieved, the retrieval order is freely chosen~\cite{Mirzaei2017, He2023}.

Storage and retrieval with a given sequence has been studied for train-yards~\cite{train-ordering} and for stack-based ship containers.
In the Block Relocation Problem (BRP)~\cite{Caserta2012} uniform loads are stored in vertical stacks where only the top load in a stack is accessible.
BRP asks to retrieve all the loads in a given order while minimizing relocations of loads between stacks.
In the spirit of this work, the BRP family addresses deviations from a planned storage/retrieval sequence.
\cite{boge2020robust} treat the retrieval priorities (where each priority refers to a batch of retrieved loads) as uncertain: the realized order may differ from the planned one by at most~$\Gamma$ pairwise inversions. %
The Stochastic Container Relocation Problem \cite{galle2018scrp} instead reveals the retrieval sequence in batches, with the order unknown within each batch.
\cite{boschma2023adp} models uncertainty for both storage and retrieval, applying approximate dynamic programming.

Nevertheless, the ordered retrieval of loads in grid settings with relocation minimization remains nascent.
A recent early effort~\cite{ICCL24} applies BRP techniques, assigning a removal direction to each load, which requires many empty aisles. %
Multi-Agent Path Finding (MAPF)~\cite{DBLP:conf/socs/SternSFK0WLA0KB19} addresses motion planning for robots in grids.
MAPF works for warehouses~\cite{li2021lifelong} typically assume aisles to increase throughput (sacrificing storage space), while we focus on maximizing capacity with sequential motion.